%% file: main_arxiv.tex
\documentclass[11pt,letterpaper]{article}

\usepackage[utf8]{inputenc} % allow utf-8 input
\usepackage[T1]{fontenc}    % use 8-bit T1 fonts
\usepackage{hyperref}       
\usepackage{url}            % simple URL typesetting
\usepackage{booktabs}       % professional-quality tables
\usepackage{xcolor}
\usepackage{amsfonts}       % blackboard math symbols
\usepackage{nicefrac}       % compact symbols for 1/2, etc.
\usepackage{microtype}      % microtypography
\usepackage{appendix}
\usepackage{amsthm}
\usepackage{amsmath}
\usepackage{amssymb}
\usepackage{enumerate}
\usepackage{algorithm}
\usepackage{algpseudocode}
\usepackage{geometry}
\usepackage{setspace}
\usepackage{xspace}
\usepackage{mathabx}
\usepackage{graphicx}
\usepackage{times}
\usepackage{tabularx, makecell, booktabs, array, longtable}
\usepackage{natbib}
\usepackage{verbatim}

\newcommand{\Prob}{\mathbb P}
\newcommand{\op}{\mathrm{op}}

\newcommand{\normop}[1]{\left\|#1\right\|_{\op}}
\newcommand{\Sc}{\widehat\Sigma}

\newcommand{\Ed}{\mathcal E_d}
\newcommand{\Ct}{\mathcal C_\ell}
\newcommand{\ot}{\widetilde O}

\DeclareMathOperator{\Lap}{Lap}
\DeclareMathOperator{\Bern}{Bernoulli}
\input{math_commands}

\input{_includes}

\makeatletter
\newif\ifhf@appendixtoc
\newcommand{\hf@appendixtocline}[2]{%
  \addtocontents{atoc}{\protect\contentsline{#1}{#2}{\thepage}{\@currentHref}\protected@file@percent}%
}
\newcommand{\hf@maybeappendixtocline}[2]{%
  \def\hf@entrytype{#1}%
  \def\hf@sectiontype{section}%
  \def\hf@subsectiontype{subsection}%
  \def\hf@subsubsectiontype{subsubsection}%
  \ifx\hf@entrytype\hf@sectiontype
    \hf@appendixtocline{#1}{#2}%
  \else\ifx\hf@entrytype\hf@subsectiontype
    \hf@appendixtocline{#1}{#2}%
  \else\ifx\hf@entrytype\hf@subsubsectiontype
    \hf@appendixtocline{#1}{#2}%
  \fi\fi\fi
}
\let\hf@oldaddcontentsline\addcontentsline
\renewcommand{\addcontentsline}[3]{%
  \hf@oldaddcontentsline{#1}{#2}{#3}%
  \ifhf@appendixtoc
    \def\hf@contentsfile{#1}%
    \def\hf@tocfile{toc}%
    \ifx\hf@contentsfile\hf@tocfile
      \hf@maybeappendixtocline{#2}{#3}%
    \fi
  \fi
}

\newcommand{\startappendixcontents}{\hf@appendixtoctrue}
\newcommand{\appendixtableofcontents}{%
  \section*{\contentsname}%
  \@starttoc{atoc}%
}
\makeatother

\title{ Improved  Private Sparse Covariance\\ Estimation with Multiscale Threshold Tests}
\author{Zihan Zhang \\ Department of CSE, HKUST \\ zihanz@cse.ust.hk}
\date{}

\begin{document}
\maketitle
\thispagestyle{empty}

\input{abstract}
%\clearpage
%\setcounter{page}{1}

\input{intro}

\input{pre}

\input{alg}
\input{pf}

\input{dis}

\subsection*{AI use statement}
GPT 6 Astra was used to search for and refine candidate choices of algorithmic parameters. The parameter search included the initial candidate-set size, the clipping thresholds, the size of the selected candidate set \(J\), and the per-scale repetition counts \(m_\ell\).  The authors take full responsibility for the final manuscript, including the mathematical claims and proofs, the algorithm and parameter choices, and the accuracy and originality of the text and references.

\bibliography{references}
\bibliographystyle{apalike}

\newpage
\appendix
\input{appendix.tex}

\end{document}

%% file: math_commands.tex
\usepackage{amsmath,amsfonts,bm}

\def\eqref#1{equation~\ref{#1}}
\def\1{\bm{1}}

\DeclareMathAlphabet{\mathsfit}{\encodingdefault}{\sfdefault}{m}{sl}
\SetMathAlphabet{\mathsfit}{bold}{\encodingdefault}{\sfdefault}{bx}{n}

%% file: _includes.tex
\usepackage{amsthm, amsmath, amssymb, mathabx}
\allowdisplaybreaks % align split page
\usepackage{graphicx}
\usepackage{enumerate}
\usepackage{pifont}
\usepackage{booktabs}
\usepackage{multirow}
\usepackage{multicol}
\usepackage{stfloats}
\usepackage{xspace}
\usepackage{thmtools}
\usepackage{thm-restate}
\usepackage{hyperref}
\usepackage{cleveref}
\crefname{appendix}{appendix}{appendices}
\Crefname{appendix}{Appendix}{Appendices}
\crefname{subappendix}{appendix}{appendices}
\Crefname{subappendix}{Appendix}{Appendices}
\crefname{subsubappendix}{appendix}{appendices}
\Crefname{subsubappendix}{Appendix}{Appendices}
\makeatletter
\let\hf@oldappendix\appendix
\renewcommand{\appendix}{%
  \hf@oldappendix
  \crefalias{section}{appendix}%
  \crefalias{subsection}{subappendix}%
  \crefalias{subsubsection}{subsubappendix}%
}
\makeatother
\crefname{ALC@unique}{line}{lines}
\Crefname{ALC@unique}{Line}{Lines}
\usepackage{anyfontsize} % customize font sizes
\usepackage{makecell}
\usepackage{hhline}
\usepackage{colortbl}
\usepackage{xpatch}
\usepackage{scalerel,stackengine}
\usepackage{sidecap}

\makeatletter
\def\@fnsymbol#1{\ensuremath{\ifcase#1\or *\or \dagger\or \ddagger\or
  \mathsection\or \mathparagraph\or \|\or \diamond \or **\or \dagger\dagger
  \or \ddagger\ddagger \else\@ctrerr\fi}}
\makeatother

\makeatletter
\newcommand{\printfnsymbol}[1]{%
  \textsuperscript{\@fnsymbol{#1}}%
}
\makeatother

\newtheorem{theorem}{Theorem}
\newtheorem{lemma}{Lemma}

\newtheorem{definition}{Definition}

\theoremstyle{definition}

\newcommand{\Roma}[1]{\uppercase\expandafter{\romannumeral#1}}

\usepackage[colorinlistoftodos, textwidth=18mm]{todonotes}

\def\shownotes{1}
\ifnum\shownotes=0
\newcommand{\todorz}[1]{}
\newcommand{\todorzout}[1]{}
\newcommand{\todossdout}[1]{}
\newcommand{\todossd}[1]{}
\else
\newcommand{\todorz}[1]{\todo[color=blue!10, inline]{\small RZ: #1}}
\newcommand{\todorzout}[1]{\todo[color=blue!10]{\scriptsize RZ: #1}}
\newcommand{\todossdout}[1]{\todo[color=red!10]{\scriptsize SSD: #1}}
\newcommand{\todossd}[1]{\todo[color=red!10, inline]{\small SSD: #1}}
\fi

%% file: abstract.tex
\begin{abstract}
We study differentially private covariance estimation in operator norm for mean-zero sub-Gaussian distributions with unknown covariance support and at most $k$ nonzero entries per row. We develop a multiscale random-threshold algorithm with sample complexity $\ot(k^2/\alpha^2+k\sqrt d/(\alpha\varepsilon))$ for $(\varepsilon,\delta)$-differential privacy and error at most $\alpha\sigma^2$, where $d$ is the dimension and $\sigma$ is a known sub-Gaussian scale. The bound improves the privacy-dependent term of the existing $\ot(k^2/\alpha^2+k^{3/2}\sqrt d/(\alpha\varepsilon))$ \citep{kumar2026curse} upper bound by a factor of $\sqrt k$, and matches the lower bound of $\widetilde{\Omega}(k^2/\alpha^2 + k\sqrt{d}/(\alpha\varepsilon))$ in its applicable parameter regime. 

Our key technical ingredient is a direct operator-norm bound on the
centered fluctuations of an ideal reconstruction, exploiting conditional
independence rather than accumulating entrywise errors across each row.
A multiscale allocation of threshold tests balances reconstruction
variance against query sensitivity. 
Together, these ingredients sharpen the trade-off between approximation
error and privacy protection, removing the additional $\sqrt{k}$ factor
from the privacy-dependent sample complexity.
\end{abstract}

%% file: intro.tex
\section{Introduction}\label{sec:intro}
Covariance estimation is a basic task in multivariate statistics and
provides the spectral information used by methods such as principal
component analysis. Its difficulty changes substantially when the
ambient dimension is comparable to or larger than the number of
observations. In this regime, the empirical covariance matrix can be
inaccurate in operator norm even when its individual entries are well
estimated. Structural assumptions are therefore essential for obtaining
useful guarantees from limited data. Covariance sparsity provides one
such assumption: each coordinate has nonzero covariance with only a
small number of other coordinates, although individual observations may
still be dense \citep{bickel2008thresholding,cai2012optimal}.

A substantial non-private literature shows how to exploit this
structure. \citet{bickel2008thresholding} establish operator-norm guarantees
for thresholding the empirical covariance, while
\citet{cai2011adaptive} develop entry-dependent thresholds that adapt to
the variability of individual covariance estimates.
\citet{cai2012optimal} complement these algorithmic developments with
minimax characterizations over sparse covariance classes. These results
demonstrate that suitable regularization can replace polynomial
dependence on the ambient dimension with dependence primarily on
sparsity and logarithmic dimension factors. They also highlight the
importance of the loss function: accurate estimation of a covariance
matrix in operator norm requires controlling the aggregate effect of
errors across its entries, rather than treating those entries as
unrelated scalar estimation problems.

Privacy introduces a different constraint. Differential privacy requires
the output distribution to remain stable when a single observation is
replaced, and this requirement applies to arbitrary neighboring
datasets, not only datasets that are typical under a statistical model
\citep{dwork2014algorithmic}. Consequently, sparsity of the population
covariance does not directly yield a sparse or low-sensitivity empirical
statistic. Even a distribution with diagonal covariance can generate
dense sample vectors and dense outer products. For example, if
$X\sim\mathcal N(0,I_d)$, then
$\mathbb E\|X\|_2^2=d$, despite the covariance having only one nonzero
entry per row. Thus, bounds based on uniformly small sample norms must
be interpreted carefully when assessing dimensional dependence.
Private thresholding methods provide an initial way to combine
sparsity with privacy \citep{wang2021differentially}, but determining the
sample cost of privacy under a natural statistical normalization
requires a more precise understanding of this distinction.

We consider covariance estimation from $n$ independent and identically
distributed observations drawn from a mean-zero, $\sigma$-sub-Gaussian
distribution on $\mathbb R^d$, where the scale $\sigma>0$ is known.
The population covariance $\Sigma$ is $k$-row-column sparse: each row
and column contains at most $k$ nonzero entries, with diagonal entries
counted toward this bound. Its support is unknown, and we impose
neither an eigengap condition nor a sparsity assumption on its
eigenvectors. Our goal is to construct an
$(\varepsilon,\delta)$-differentially private estimator
$\widehat\Sigma$ that approximates $\Sigma$ in operator norm with
high probability:
\begin{equation}
	\label{eq:intro-accuracy}
	\Pr\!\left(
	\|\widehat\Sigma-\Sigma\|_{\mathrm{op}}
	\le \alpha\sigma^2
	\right)
	\ge 1-\beta,
\end{equation}
where the probability is taken jointly over the observations and the
internal randomness of the estimator.
 In the
non-private case, thresholding yields the benchmark sample complexity
$\widetilde{O}(k^2/\alpha^2)$\footnote{ Throughout, the tilde notation suppresses
logarithmic factors in the problem parameters.}
\citep{bickel2008thresholding}.
\citet[Theorem~1]{kumar2026curse} obtain the private upper bound $
	\widetilde O\!\left(
	k^2/\alpha^2
	+k^{3/2}\sqrt d/(\alpha\varepsilon)
	\right).$
Their Theorem~2 gives the lower bound $
	\widetilde\Omega\!\left(
	k^2/\alpha^2
	+k\sqrt d/(\alpha\varepsilon)
	\right)$
when $  \log (d/\alpha)\lesssim k \lesssim d^{1/2-\lambda}$\footnote{The non-private lower bound holds for $\frac{k^2 \log d}{\alpha^2} \lesssim c_\zeta d^{1/\zeta}$ for any $\zeta>1$} and $\delta  \lesssim\ \varepsilon^2/d^2$.
These bounds identify a $\sqrt d$ privacy cost in that regime, but
leave a $\sqrt k$ gap in the privacy term. This raises a natural question: \emph{Is it possible to remove the additional sparsity factor
	$\sqrt{k}$ from the privacy term and thereby match the lower bound
	up to logarithmic factors?}
	
\paragraph{Our contribution. }
We answer this question affirmatively by developing a multiscale
randomized-threshold mechanism that achieves
\eqref{eq:intro-accuracy} using
\begin{equation}
	\label{eq:intro-result}
	\widetilde O\!\left(
	\frac{k^2}{\alpha^2}
	+\frac{k\sqrt d}{\alpha\varepsilon}
	\right)
\end{equation}
samples. The mechanism satisfies
$(\varepsilon,\delta)$-differential privacy under replace-one
adjacency for every input dataset.
Compared with \citep{kumar2026curse}, our bound preserves the
non-private sample complexity term and improves the privacy term
by a factor of $\sqrt{k}$, without imposing additional structural
assumptions. Consequently, it closes the $\sqrt{k}$ gap between
the upper and lower bounds of \citet{kumar2026curse}, up to
logarithmic factors when $  \log (d/\alpha)\lesssim k \lesssim d^{1/2-\lambda}$ for some constant $\lambda>0$, and $\delta \lesssim \varepsilon^2/d^2$.

\paragraph{Technical ingredients.} Let $S_{\mathrm{clip}}$ denote the empirical covariance matrix
computed from clipped observations, and let $\widehat\Sigma$
denote the private estimator.
Our key technical observation is that the randomized reconstruction,
rather than $S_{\mathrm{clip}}$ itself, inherits the row sparsity
of the population covariance on a high-probability empirical
concentration event.
Specifically, conditional on $S_{\mathrm{clip}}$, the ideal
estimator $\Sigma^{\mathrm{id}}$ is a sum of independent random
matrices whose contributions vanish outside the support of $\Sigma$.
Writing
$\overline\Sigma(S_{\mathrm{clip}})
:=\mathbb E[\Sigma^{\mathrm{id}}\mid S_{\mathrm{clip}}]$,
we apply matrix concentration directly to
$\|\Sigma^{\mathrm{id}}-\overline\Sigma(S_{\mathrm{clip}})\|_{\mathrm{op}}$,
rather than first bounding the entrywise fluctuations and then
using
$k\|\Sigma^{\mathrm{id}}-\overline\Sigma(S_{\mathrm{clip}})\|_{\max}$.
This allows us to exploit the centered, independent structure of
the random contributions instead of accumulating their absolute
errors across each row.
We separately control the error of the conditional mean,
$\|\overline\Sigma(S_{\mathrm{clip}})-\Sigma\|_{\mathrm{op}}$,
using covariance sparsity, and transfer the resulting accuracy
guarantee to $\widehat\Sigma$ through a high-probability coupling.
		
Building on this observation, we design our mechanism around
three key ingredients.
\begin{itemize}
	\item[(i)] \textbf{A noise-corrected randomized representation.}
	Rather than releasing covariance entries directly, we construct a
	multiscale, randomly weighted representation of the clipped empirical
	covariance. 
	A noise-corrected scalar integral identity gives the corresponding
	ideal estimator a small truncation bias. Crucially, reweighting
	compensates in expectation for randomly omitted contributions,
	without requiring exact support recovery.
	
	\item[(ii)] \textbf{Scale-dependent allocation and direct matrix
		concentration.}
	We allocate the expected number of active tests per position
	quadratically in the threshold scale, while both normalized-query
	sensitivity and output weights decrease inversely with that scale.
	This balances the squared-sensitivity and second-moment terms across
	scales. Conditional on the empirical matrix, the ideal estimator is
	a sum of independent random matrices, allowing us to control its
	fluctuations directly in operator norm using matrix Bernstein
	\citep{tropp2012user}, rather than accumulating entrywise
	random-error bounds across each row. 
	
	\item[(iii)] \textbf{Private quotas and high-probability coupling.}
	To ensure privacy on arbitrary datasets, we enforce public,
	scale-dependent quotas through private oneshot top-$k$ selection
	\citep{qiao2021oneshot}. These quotas provide a uniform sensitivity
	bound for the subsequent query vector for every fixed selected
	label set. Under the statistical model, we prove that selection
	retains every potentially nonzero contribution with high probability.
	Using the same threshold tests and gating noise, the actual output
	then coincides exactly with the ideal estimator on this event.
	We can therefore transfer the ideal estimator's accuracy without
	conditioning its concentration argument on successful selection,
	while establishing privacy separately for all adjacent datasets.
\end{itemize}

\paragraph{Related work.}
Research on covariance estimation spans several complementary directions.
Structural regularization and minimax theory characterize the benefits of
covariance structure \citep{cai2012optimal,cai2016estimating}.
Private matrix release studies privacy--accuracy trade-offs for bounded
records and data-dependent error guarantees
\citep{amin2019differentially,dong2022differentially}.
Private distribution learning and stable estimation address Gaussian and
sub-Gaussian mean and covariance estimation through preconditioning and
affine-invariant procedures
\citep{kamath2019privately,brown2023fast}.
Dedicated work on private sparse covariance estimation combines sparsity
with privacy constraints \citep{wang2021differentially}.
These directions differ in their structural assumptions, input
normalizations, privacy definitions, and estimation losses.
Appendix~\ref{app:rel} provides a broader discussion.

%% file: pre.tex
\section{Preliminaries}\label{sec:model}
\paragraph{Sparse covariance matrix.}
We study the sparse covariance estimation problem of
\citet[Model~1 and Problem~1]{kumar2026curse},
without an eigengap assumption.
Let $D=\{X_1,\ldots,X_n\}$ consist of independent and
identically distributed samples from an unknown mean-zero
distribution $P$ on $\mathbb R^d$.
We assume that $P$ is $\sigma$-sub-Gaussian: for a known
scale $\sigma>0$,
\begin{equation}
	\label{eq:model}
	\mathbb E_{X\sim P}\!\left[\exp(u^\top X)\right]
	\le
	\exp\!\left(\frac{\sigma^2\|u\|_2^2}{2}\right)
	\qquad
	\text{for every }u\in\mathbb R^d.
\end{equation}
The population covariance
$\Sigma=\mathbb E_{X\sim P}[XX^\top]$ is
\emph{$k$-row-column sparse}: each row and each column
contains at most $k$ nonzero entries, with nonzero
diagonal entries counted toward this bound.
Since $\Sigma$ is symmetric, the row and column
sparsity conditions are equivalent.
The support of $\Sigma$ is unknown to the learner.

The sub-Gaussian condition implies
$v^\top\Sigma v=\mathbb E[(v^\top X)^2]
\le \sigma^2\|v\|_2^2$ for every $v\in\mathbb R^d$,
and hence $0\preceq\Sigma\preceq\sigma^2 I_d$.
Importantly, sparsity is imposed only on the population
covariance; the observations, their outer products, and
the empirical covariance matrix need not be sparse.

\begin{definition}[Differential privacy]
	\label{def:dp}
	Let $\mathcal X$ be the data domain, and let
	$\mathcal Y$ be a measurable output space. 
	Two datasets $D=\{X_1,\ldots,X_n\}$ and
	$D'=\{X_1',\ldots,X_n'\}$ in $\mathcal X^n$ are adjacent
	under replace-one adjacency, denoted $D\sim D'$, if they
	differ in at most one sample.
	For $\varepsilon\ge 0$ and $\delta\in[0,1)$, a randomized
	mechanism $\mathcal M:\mathcal X^n\to\mathcal Y$ satisfies
	$(\varepsilon,\delta)$-differential privacy if, for every
	adjacent pair $D\sim D'$ and every measurable event
	$E\subset  \mathcal{Y}$,
	\[
	\Pr\!\left(\mathcal M(D)\in E\right)
	\le
	e^\varepsilon
	\Pr\!\left(\mathcal M(D')\in E\right)
	+\delta,
	\]
	where the probabilities are taken over the internal
	randomness of $\mathcal M$, with $D$ and $D'$ held fixed.
\end{definition}

\paragraph{Private estimation objective.} Let $\|\cdot \|_{\mathrm{op}}$ denote the operator norm, i.e.,
	 $\|A\|_{\mathrm{op}} = \max_{\|x\|_2 \leq 1}\|Ax\|_2$.  
The goal is to design a randomized mechanism
$\mathcal M:(\mathbb R^d)^n\to\mathbb S^d$,
where $\mathbb S^d$ denotes the space of real symmetric
$d\times d$ matrices, that satisfies
$(\varepsilon,\delta)$-differential privacy under
replace-one adjacency (Definition~\ref{def:dp}) and
returns $\widehat\Sigma=\mathcal M(D)$ satisfying
\begin{equation}
	\label{eq:accuracy}
	\Pr\!\left(
	\|\widehat\Sigma-\Sigma\|_{\mathrm{op}}
	\le \alpha\sigma^2
	\right)
	\ge 1-\beta
\end{equation}
for every distribution $P$ obeying the assumptions above.
Here the probability is taken jointly over the sampled
dataset and the internal randomness of $\mathcal M$.
The distributional assumptions are required only for
accuracy: the privacy guarantee must hold for every pair
of adjacent datasets in $(\mathbb R^d)^n$.
We seek to minimize the sample size $n$ needed to meet
these privacy and accuracy requirements.

\paragraph{Additional notations.}
Let $
\mathcal E_d
:=\{(i,j):1\le i\le j\le d\}$.
Thus, $\mathcal E_d$ indexes the upper-triangular entries,
including the diagonal.
For the standard basis vectors $e_1,\ldots,e_d$, define
\[
B_{ii}:=e_i e_i^\top,
\qquad
B_{ij}:=e_i e_j^\top+e_j e_i^\top
\quad (i<j).
\]
For $e=(i,j)\in\mathcal E_d$, write 
$B_e:=B_{ij}$. Then $\{B_e\}_{e\in\mathcal E_d}$ forms a basis for the space
$\mathbb{S}^d$ of real symmetric $d\times d$ matrices.
%Every symmetric matrix $Q\in\mathbb S^d$ then admitsthe representation $ Q=\sum_{e\in\mathcal E_d}Q_e B_e.$Each off-diagonal basis matrix accounts for both symmetric entries, whereas a diagonal basis matrix accounts for a single entry.

We write
$\|Q\|_{\max}:=\max_{i,j}|Q_{ij}|$
for the entrywise maximum norm.
For any symmetric matrix $Q$ with at most $k$ nonzero
entries in each row, we have that  $
	\|Q\|_{\mathrm{op}}
	\le
	\max_i\sum_j |Q_{ij}|
	\le
	k\|Q\|_{\max}$.

Throughout, we assume $
d\ge 2$, $
1\le k\le d$, $
0<\alpha\le \frac14$, 
$0<\varepsilon\le 1$ and $
0<\delta,\beta\le \frac1{10}$. All logarithms are natural unless marked $\log_2$.   We use $\ot(\cdot)$ to hide the logarithmic factors of $(d,k,1/\varepsilon, 1/\alpha, 1/\beta,1/\delta)$.
We use $[N]$ to denote the set $\{1,2,\ldots, N\}$ for an integer $N$.  
For a threshold $z\ge 0$, define the scalar clipping function by $
\mathrm{Clip}_z(x)
:=\max\{\min\{x,z\},-z\},$ for $x\in\mathbb R.$
For a vector $X\in\mathbb R^d$, we apply clipping coordinatewise: $
[\mathrm{Clip}_z(X)](j)
:=\mathrm{Clip}_z(X(j)),  j=1,\ldots,d.$ For $b>0$, let $\operatorname{Lap}(b)$ denote the
zero-mean Laplace distribution with scale parameter $b$
and probability density $
f_b(x)=\frac{1}{2b}\exp\!\left(-\frac{|x|}{b}\right)$ for $
x\in\mathbb R.$
We write $\mathcal N(\mu,\sigma^2)$ for the uni-variate
Gaussian distribution with mean $\mu\in\mathbb R$ and
variance $\sigma^2>0$.

%% file: alg.tex
\section{The multiscale threshold mechanism}\label{sec:method}

In this section, we present our private covariance estimation mechanism, summarized in Algorithm~\ref{alg:main}, and introduce its parameters $L$, $R$, $r$, and ${p_{\ell}, t_{\ell}, b_{\ell}, s_{\ell}}_{\ell=0}^{L}$.

As in \citet[Algorithm~1]{kumar2026curse}, the mechanism consists
of private selection followed by reconstruction using fresh,
independent noise.
Their method performs row-wise private top-$k$ selection and
constructs a symmetric estimate from independently perturbed
entry values, with the diagonal handled separately.
Our method instead selects a set  $\mathcal J$ of randomized
threshold tests. Each selected, active test produces a value by applying a gate to a noisy normalized query.
This value is then multiplied by a public weight and added
to the corresponding covariance entry.  This design enables direct control of reconstruction fluctuations in operator norm, without  converting an entrywise bound into an operator-norm bound.

\paragraph{Pre-process: clipping step.} Given the dataset $D= \{ X_1,X_2,\ldots, X_n\}$, we first clip $D$ up to a proper threshold.  Define
 $R=\sigma\sqrt{2\log(40nd/\beta)}$. We set $Y_i  = \mathrm{Clip}_R(X_i)$ and  $S(D) = \frac{1}{n} \sum_{i=1}^n Y_i Y_i^{\top}$. When $D$ is clear from the context, we write $S$ for $S(D)$. In the rest of the algorithm, we will use $D_{\mathrm{clip}}=\{ Y_1,Y_2,\ldots, Y_n\}$ to construct the private estimator. The probability that any observation is altered by clipping is
 accounted for in the overall failure probability $\beta$.

\begin{algorithm}[t]
	\caption{Multiscale private covariance estimation}\label{alg:main}
	\begin{algorithmic}[1]
		\Require Dataset $D = \{X_1,X_2,\ldots, X_n\}$; public parameters $d,k,\sigma,\alpha,\varepsilon,\delta,\beta$.
		\State $Y_i \leftarrow \mathrm{Clip}_{R}(X_i)$ for $i = 1,2,\ldots. n$;
		\State Set $S\leftarrow\frac{1}{n}\sum_{i=1}^nY_iY_i^\top$.
		\For{$c = (e,\ell,h)\in \mathcal{C}$}
			\State Draw independent public $U_c \sim \operatorname{Bern}(p_{\ell})$, $T_c\sim \operatorname{Unif}[t_{\ell}, 2t_{\ell}]$.\label{line:pub_ran}
		\EndFor
		\For{$\ell=0,\ldots,L-1$}
		\If{$s_\ell<M_\ell$}
		\For{every $c = (e,\ell, h)\in \mathcal{C}_{\ell}$}
		\State Draw independent $\zeta_c\sim\Lap(b_\ell)$
		\State $q_c \leftarrow ( |S_{e}|/T_c -1) U_c - 2(1-U_c)$
		\EndFor
\State Select $J_\ell \in
\operatorname*{arg\,max}_{\substack{J\subseteq\mathcal C_\ell\\
		|J|=s_\ell}}
\sum_{c\in J}(q_c+\zeta_c)$;
		\Else
		\State $J_\ell\gets\Ct$.
		\EndIf
		\EndFor
		\State $\Sc\gets0_{d\times d}$.
		\For{each $c = (e,\ell, h)\in J = \bigcup_{\ell=0}^{L-1}J_\ell$, in a fixed public order}
		\If{$U_c=1$}
		\State Draw a fresh independent $\check{Z}_c\sim N(0,r^2)$ and set $Z_c  = \mathrm{Clip}_{1/4}(\check{Z}_c)$.
		\State $\Sc\gets\Sc+\frac{1}{\kappa\rho t_{\ell}}\cdot \psi\left(\frac{S_e}{T_c}+Z_c\right)B_{e}$.
		\EndIf
		\EndFor
		\State \Return $\Sc$
	\end{algorithmic}
\end{algorithm}

\paragraph{Stage I: Private candidate selection.}
% H  = \log(80dL/\beta)
Fix two constants $C_s =64$ and $C_{\rho} = 4096$. Define $t_0  =\frac{\alpha\sigma^2}{256k}$ to be the minimal scale, and  $L=\left\lceil\log_2\frac{512k}{\alpha}\right\rceil$ to be total number of scales. Let $t_\ell:=2^\ell t_0$ for
$\ell=0,\ldots,L$. By the choice of $L$, the largest threshold
satisfies $2\sigma^2\le t_L<4\sigma^2$. We further choose $\rho = \frac{C_{\rho}L k\log(80dL/\beta) }{\alpha^2\sigma^4}$,  $\mu_{\ell} = \rho t_{\ell}^2$, $m_{\ell} = \left\lceil \mu_{\ell}\right \rceil$, $p_{\ell} = \mu_{\ell}/m_{\ell}$ and $M_{\ell} = \frac{d(d+1)}{2}\cdot  m_{\ell}$.

We index each candidate threshold test by a tuple $c=(e,\ell,h)$, where
$e\in\mathcal E_d$ specifies the matrix position,
$0\leq \ell \leq L-1$ specifies the scale level, and
$h\in[m_\ell]$ is the repetition index.
Thus, for each position--scale pair $(e,\ell)$, we create
$m_\ell$ candidates, each using independently generated
randomness for activation, threshold generation, and
subsequent private selection.

Formally, the candidate set is $\mathcal{C} =\{(e,\ell,h):e\in\Ed, 0 \leq \ell \leq L-1, 1\le h\le m_\ell\}$.  Let $\mathcal{C}_{\ell}$ denote the candidate set at level $\ell$, i.e., $\mathcal{C}_{\ell} = \{ (e,\ell, h): e\in \Ed, 1\leq h \leq m_{\ell}\}$. By definition, the size of  $\mathcal{C}_{\ell}$ is  $M_{\ell} = \frac{d(d+1)}{2}\cdot m_{\ell}$.

For every candidate $c=(e,\ell,h)$,  the algorithm generates independent public randomness (see Line~\ref{line:pub_ran} Algorithm~\ref{alg:main}) 
$U_c\sim\Bern(p_\ell)$ and $
T_c\sim \operatorname{Unif}[t_{\ell}, 2t_{\ell}]$,
and computes $q_c=(|S_{e}|/T_c -1) U_c - 2(1-U_c)$. 
Here, $U_c$ indicates whether candidate $c$ is active,
$T_c$ is its random threshold, and $q_c$ is its
unperturbed selection score.

Choose $H = \log(80dL/\beta)$,  $ \Delta=2R^2/n$, $ \varepsilon'=\frac{\varepsilon}{8L}$ and $
\delta'=\frac{\delta}{32L}$. Define $s_\ell=\min\!\left\{M_\ell,
\left\lceil C_s\left[d\rho\min\{kt_\ell^2,\sigma^4\}+H\right]\right\rceil\right\}$.
If $s_\ell<M_\ell$, add independent $\zeta_c\sim\Lap(b_\ell)$, where
\begin{equation}\label{eq:laplace}
b_\ell=\frac{8(\Delta/t_\ell)\sqrt{s_\ell\log(M_\ell/\delta')}}{\varepsilon'},
\end{equation}
and set
		\begin{equation}
			J_\ell \in
			\operatorname*{arg\,max}_{\substack{
					J\subseteq\mathcal C_\ell\\
					|J|=s_\ell
			}}
			\sum_{c\in J}(q_c+\zeta_c).
		\end{equation}
In words,  $J_\ell$ denotes the set of candidates with the $s_\ell$ largest noisy scores $q_c+\zeta_c$  among
$\mathcal{C}_{\ell}$\footnote{The independent, continuous Laplace perturbations
	ensure that the noisy scores are distinct almost surely,
	so $J_\ell$ is uniquely determined.
	Any ties arising from finite-precision arithmetic can be
	resolved using a fixed, data-independent ordering of the
	candidate labels.}.
If $s_\ell=M_\ell$, the algorithm directly sets $J_{\ell} =\mathcal{C}_{\ell}$.
In this way, we obtain $J = \cup_{\ell=0}^{L-1} J_{\ell}$ as  the set of selected private candidate tests. 

At each level, the number of active candidates is chosen to balance variance reduction against privacy loss. Distributing the reconstruction weight across more candidates can reduce the estimator’s variance conditional on $S$, but may incur a greater privacy cost during reconstruction. The selection quotas are therefore calibrated to achieve low conditional variance while remaining within the prescribed privacy budget.

\paragraph{Stage II: Private covariance reconstruction.}
Define
\begin{equation}\label{eq:gaussian}
A=\sum_{\ell=0}^{L-1}\frac{s_\ell}{t_\ell^2},\qquad
S_\star=\sum_{\ell=0}^{L-1}s_\ell,\qquad
r=\frac{8\Delta\sqrt A}{\varepsilon}\sqrt{\log\frac{32}{\delta}}.
\end{equation}
 Let $\nu_r$ denote the distribution of $Z = \mathrm{Clip}_{1/4}(\check{Z})$ for $\check{Z}\sim \mathcal{N}(0,r^2)$, and set $
\kappa=\mathbb{E}_{Z\sim\nu_r}\left[\frac1{1-Z}\right]\in \left[ 1,\frac{4}{3}\right]$. 
Define 
$\psi(x)=\textbf{1}\{x>1\}-\textbf{1}\{x<-1\}.$

For each selected candidate $c\in J$, draw a fresh independent $Z_c\sim\nu_r$ and return
\begin{equation}\label{eq:output}
\Sc=\sum_{\ell=0}^{L-1}\sum_{\substack{c=(e,\ell, h)\in J_\ell}}
\frac{\psi(S_{e}/T_c+Z_c)\cdot U_c}{\kappa\rho t_{\ell}}\,B_{e}.
\end{equation}

In this reconstruction, for each $e\in\mathcal{E}_d$, $\hat{\Sigma}_e$ is a noisy weighted combination of the outcomes of active candidate tests across scales. The noises are carefully chosen to ensure that $\hat{\Sigma}$ concentrates around its conditional mean $\mathbb{E}[ \hat{\Sigma}|S]$ within the desired error tolerance.

%% file: pf.tex
\section{Analysis of Approximation Error and Privacy}\label{sec:proof}

We first present the main result under Algorithm~\ref{alg:main}. Set $M=\sum_{\ell=0}^{L-1}M_\ell$ and $\Gamma=\log(80LM/(\beta\delta))$.

\begin{theorem}[Privacy and accuracy]\label{thm:main}
There exists a universal constant $C>0$ such that, with the
public parameters defined in Section~\ref{sec:method}, whenever
\begin{equation}\label{eq:sample}
    n\ge C\!\left[
    \frac{k^2}{\alpha^2}\log\frac{80d}{\beta}
    +\frac{k\sqrt d}{\alpha\varepsilon}
    \log\frac{80nd}{\beta}\,
    L^{3/2}\sqrt H\,\Gamma^{3/2}
    \right],
\end{equation}
Algorithm~\ref{alg:main} is $(\varepsilon,\delta)$-differentially
private on every input dataset $D\in(\mathbb{R}^d)^n$.
Moreover, under the stated distributional assumptions, $
    \Prob\!\left(\normop{\Sc-\Sigma}\le\alpha\sigma^2\right)
    \ge 1-\beta$, 
where the probability is taken over the sampling of $D$ and
all randomness of the algorithm.
\end{theorem}
Recalling that $L = \left\lceil \log_2 \frac{512k}{\alpha}\right \rceil$ and $H= \log(80dL/\beta)$, Algorithm~\ref{alg:main} ensures a sample complexity bound of $\tilde{O}\left( k^2/\alpha^2 + k\sqrt{d}/(\alpha\varepsilon)\right) $. This improves the privacy-dependent term by a factor of $\sqrt{k}$ compared to the upper bound in \citep{kumar2026curse}, and matches
the lower bound of $\tilde{\Omega}( k^2/\alpha^2 + k\sqrt{d}/(\alpha\varepsilon) )$ \citep{kumar2026curse} in its applicable parameter regime.

In the rest of this section, we present the proof sketch of Theorem~\ref{thm:main}. 
We postpone the supporting lemmas and proofs to Appendices~\ref{sec:accuracy:lemmas} and \ref{app:privacy}.  Assume that the condition in \eqref{eq:sample} holds, with $C$ to be specified later.

\subsection{Bound of Approximation Error}

We work under the sampling model and the sample-size condition of the
main theorem. Recall that $S$ is the clipped empirical covariance matrix.
For the accuracy analysis, introduce an ideal estimator
$\Sigma^{\mathrm{id}}$ that includes all candidates in $\mathcal C$.
Generate independent reconstruction noises $Z_c\sim\nu_r$ for every
$c\in\mathcal C$, independently of the data, public candidate randomness,
and selection noises, and define
\begin{equation}
\label{eq:accuracy:ideal}
\begin{aligned}
\Sigma^{\mathrm{id}}
&:=
\sum_{\ell=0}^{L-1}
\sum_{c=(e,\ell,h)\in\mathcal C_\ell}
\frac{U_c}{\kappa\rho t_\ell}
\psi\!\left(\frac{S_e}{T_c}+Z_c\right)B_e,\qquad 
\overline\Sigma(S)
:=
\mathbb E\!\left[\Sigma^{\mathrm{id}}\mid S\right].
\end{aligned}
\end{equation}
The conditional expectation averages over both the public candidate
randomness and the reconstruction noises. The actual mechanism uses
the same contributions, but only for selected candidates. Since
selection does not inspect the reconstruction noises, generating them
in advance and using them on the selected labels preserves the
distribution of the actual output.

We decompose the estimation error into four terms:
\begin{equation}
\label{eq:accuracy:decomposition}
\begin{aligned}
\widehat\Sigma-\Sigma
={}&
\underbrace{\widehat\Sigma-\Sigma^{\mathrm{id}}}
_{\text{selection discrepancy}}
+
\underbrace{\Sigma^{\mathrm{id}}-\overline\Sigma(S)}
_{\text{reconstruction fluctuation}} + 
\underbrace{\overline\Sigma(S)-S}
_{\text{reconstruction bias}}
+
\underbrace{S-\Sigma}
_{\text{sampling and clipping error}}.
\end{aligned}
\end{equation}
Define
\[
\mathcal E_{\mathrm{cov}}
=
\left\{
\|S-\Sigma\|_{\max}\le\frac{t_0}{16}
\right\},
\qquad
\mathcal S=\{e\in \mathcal{E}_d:\Sigma_e\ne0\}.
\]

By Bernstein's inequality (Lemma~\ref{lem:scalar-bernstein}),
the choice of clipping radius $R$ makes the probability that any observed
coordinate is altered at most $\beta/20$. By choosing $C$ large enough,  
Lemma~\ref{lem:accuracy:empirical} therefore gives $
\Pr(\mathcal E_{\mathrm{cov}}^c)\le \beta/10.$
In particular, $\|S-\Sigma\|_{\max}\le t_0/16$ with the required
probability. Clipping is accounted for through its failure probability
rather than through an additional deterministic bias term.

We analyze the four terms in \eqref{eq:accuracy:decomposition} separately, then combine the last two before
using population sparsity to obtain an operator-norm bound.

\paragraph{Selection discrepancy.}
The selection step imposes a deterministic query budget for privacy,
but under the statistical model it does not change the ideal
reconstruction with high probability. Bounded reconstruction noise
implies that an active candidate can contribute only if
$|S_e|/T_c>3/4$. On $\mathcal E_{\mathrm{cov}}$ for some $c =(e,\ell, h)$, quota coverage and
score separation ensure that every such candidate is retained, except
on a small-probability event. Under the shared-noise coupling, every
omitted ideal contribution is therefore exactly zero.
Lemma~\ref{lem:accuracy:coupling} gives
\begin{equation}
\label{eq:accuracy:coupling-bound}
\Pr\!\left(
\mathcal E_{\mathrm{cov}}
\cap\{\widehat\Sigma\ne\Sigma^{\mathrm{id}}\}
\right)
\le\frac{\beta}{20}.
\end{equation}
This is a model-dependent accuracy statement, not a uniform guarantee
for every fixed input. Selection remains necessary for the existing
privacy analysis because it bounds the number of protected queries
even on datasets whose empirical covariance has many large entries.

\paragraph{Reconstruction fluctuation.}
Fix an empirical matrix $S$ satisfying $\mathcal E_{\mathrm{cov}}$.
Conditional on this $S$, the ideal candidate contributions are
independent and vanish outside $\mathcal S$. The quadratic allocation
of active candidates across scales balances their inverse-scale
weights: each position contributes at most $1/\rho$ to the
second-moment bound at each scale. Since each row has at most $k$
support positions, the conditional matrix variance is at most
$kL/\rho = \widetilde{O}\left(\alpha^2\sigma^4\right)$, while each centered contribution has operator norm at most
$2/(\rho t_0) = \widetilde{O}(\alpha \sigma^2)$. Lemma~\ref{lem:accuracy:matrix} therefore gives
\begin{equation}
\label{eq:accuracy:fluctuation}
\Pr\!\left(
\left.
\|\Sigma^{\mathrm{id}}-\overline\Sigma(S)\|_{\mathrm{op}}
>
\sqrt{\frac{2kL\log(40d/\beta)}{\rho}}+\frac{4\log(40d/\beta)}{3\rho t_0}
\,\right|\,S
\right)
\le\frac{\beta}{20}.
\end{equation}
The choices of $\rho$ and $t_0$ make the displayed threshold at most
$\alpha\sigma^2/2$. This argument controls the centered reconstruction
directly in operator norm; it does not first bound entrywise
fluctuations and then multiply by $k$. 

\paragraph{Reconstruction bias.}
The activation rates, random thresholds, and public weights are
calibrated so that the conditional mean estimates a threshold integral:
\[
\overline\Sigma_e(S)
= \mathbb{E}[\Sigma^{\mathrm{id}}_e|S] = 
\frac{1}{\kappa}
\int_{t_0}^{t_L}
\mathbb E_Z\psi(S_e/t+Z)\,dt.
\]
The noise-corrected identity
\[
\frac{1}{\kappa}
\int_0^\infty
\mathbb E_Z\psi(a/t+Z)\,dt=a,
\qquad a\in\mathbb R,
\]
shows that random activation and noisy gating introduce no additional
uncorrected multiplicative bias. On $\mathcal E_{\mathrm{cov}}$, the
largest threshold $t_L$ is sufficiently large that the integrand
vanishes above $t_L$. The discrepancy from $S_e$ therefore comes
entirely from omitting thresholds below $t_0$.
Lemma~\ref{lem:accuracy:mean} yields
\begin{equation}
\label{eq:accuracy:bias-max}
\|\overline\Sigma(S)-S\|_{\max}\le t_0.
\end{equation}
Thus, the systematic reconstruction error is a lower-threshold
truncation error, whereas random omissions are accounted for by the
centered fluctuation term.

\paragraph{Combining the bounds.}
To convert the entrywise bounds for the last two terms into an
operator-norm bound, we first add them and then use support
preservation. On $\mathcal E_{\mathrm{cov}}$, a true-zero position
satisfies
\[
\left|\frac{S_e}{T_c}+Z_c\right|
\le\frac{1}{16}+\frac{1}{4}<1,
\]
so every ideal contribution at that position vanishes. Thus
$\overline\Sigma(S)-\Sigma$ has at most $k$ nonzero entries per row,
even though $\overline\Sigma(S)-S$ and $S-\Sigma$ may individually be
dense. It follows that
\begin{equation}
\label{eq:accuracy:bias-op}
\begin{aligned}
\|\overline\Sigma(S)-\Sigma\|_{\mathrm{op}}
&\le k\|\overline\Sigma(S)-\Sigma\|_{\max}
\\
&\le k\left(
\|\overline\Sigma(S)-S\|_{\max}
+\|S-\Sigma\|_{\max}
\right) \le\frac{17}{16}kt_0
=\frac{17}{4096}\alpha\sigma^2.
\end{aligned}
\end{equation}
Consequently, whenever the empirical, coupling, and matrix
concentration conclusions hold,
\begin{equation}
\label{eq:accuracy:final-error}
\begin{aligned}
\|\widehat\Sigma-\Sigma\|_{\mathrm{op}}
&\le
\|\widehat\Sigma-\Sigma^{\mathrm{id}}\|_{\mathrm{op}}
+
\|\Sigma^{\mathrm{id}}-\overline\Sigma(S)\|_{\mathrm{op}}
+
\|\overline\Sigma(S)-\Sigma\|_{\mathrm{op}}
<\alpha\sigma^2.
\end{aligned}
\end{equation}
Define $
\mathcal E_{\mathrm{mat}}
=
\left\{
\|\Sigma^{\mathrm{id}}-\overline\Sigma(S)\|_{\mathrm{op}}
\le\frac{\alpha\sigma^2}{2}
\right\}$. 
By \eqref{eq:accuracy:fluctuation},
\[
\Pr(\mathcal E_{\mathrm{cov}}\cap\mathcal E_{\mathrm{mat}}^c)
=
\mathbb E\!\left[
\mathbf 1_{\mathcal E_{\mathrm{cov}}}
\Pr(\mathcal E_{\mathrm{mat}}^c\mid S)
\right]
\le\frac{\beta}{20}.
\]
A union bound therefore gives
\begin{equation}
\label{eq:accuracy:failure}
\begin{aligned}
\Pr\!\left(
\|\widehat\Sigma-\Sigma\|_{\mathrm{op}}>\alpha\sigma^2
\right)
&\le
\Pr(\mathcal E_{\mathrm{cov}}^c)
+
\Pr\!\left(
\mathcal E_{\mathrm{cov}}
\cap\{\widehat\Sigma\ne\Sigma^{\mathrm{id}}\}
\right)+
\Pr(\mathcal E_{\mathrm{cov}}\cap\mathcal E_{\mathrm{mat}}^c)
\le
\frac{\beta}{5}<\beta.
\end{aligned}
\end{equation}

\subsection{The differential privacy guarantee}
\label{sec:privacy}

We then establish the privacy guarantee of our mechanism for
every neighboring pair $D\sim D'$ under replace-one adjacency. We first fix all
public activations and thresholds $\{(U_c,T_c)\}_{c\in\mathcal C}$ and
prove a privacy guarantee uniform over their realizations.

Recall
\[
\Delta=\frac{2R^2}{n},
\qquad
A=\sum_{\ell=0}^{L-1}\frac{s_\ell}{t_\ell^2},
\qquad
S_\star=\sum_{\ell=0}^{L-1}s_\ell,
\qquad
r=\frac{8\Delta\sqrt A}{\varepsilon}
\sqrt{\log\frac{32}{\delta}}.
\]
Write $
\eta:=2S_\star\exp\!\left(-\frac{1}{32r^2}\right)$. For a sufficiently large choice of $C$ in \eqref{eq:sample},
Lemma~\ref{lem:privacy:safety} yields $
    r\le \frac{1}{8\sqrt{\Gamma}}$ and $
    (1+e^{\varepsilon/4})\eta\le \frac{\delta}{16}$.

\paragraph{Private selection.}
Clipping controls the change in every empirical entry on every input:
\[
|S_e(D)-S_e(D')|\le\Delta.
\]
An active score consequently has sensitivity at most
$\Delta/t_\ell$, while an inactive score is the constant $-2$ and has
sensitivity zero. Thus the score vector at scale $\ell$ has
$L_\infty$ sensitivity at most $\Delta/t_\ell$. Applying the oneshot
Laplace selection theorem with the prescribed scale $b_\ell$ makes the
unordered selected label set $J_\ell$ to be 
$(\varepsilon',\delta')$-differentially private, where
$\varepsilon'=\varepsilon/(8L)$ and $\delta'=\delta/(32L)$.
A scale with $s_\ell=M_\ell$ returns all labels deterministically.
Composition over the scales therefore gives
\begin{equation}
\label{eq:privacy:selection}
J=(J_0,\ldots,J_{L-1})
\quad\text{is}\quad
(\varepsilon/8,\delta/32)\text{-differentially private}.
\end{equation}
Lemma~\ref{lem:privacy:selection} provides the full proof of \eqref{eq:privacy:selection}. This step
protects the data-dependent choice of which candidate tests to evaluate.

\paragraph{Uniform sensitivity after selection.}
For any fixed possible label collection $j=(j_0,\ldots,j_{L-1})$ with
$|j_\ell|=s_\ell$, arrange its labels in a public order and set
\[
F_j(D):=
\left(U_cS_{e}(D)/T_c\right)_{c=(e,\ell, h)\in j}
\in\mathbb R^{S_\star}.
\]
By Lemma~\ref{lem:privacy:query}, the public quotas bound the squared sensitivity of this entire vector:
\begin{equation}
\label{eq:privacy:query-sensitivity}
\begin{aligned}
\|F_j(D)-F_j(D')\|_2^2
&\le
\sum_{\ell=0}^{L-1}\sum_{c\in j_\ell}
\frac{\Delta^2}{t_\ell^2}=\Delta^2\sum_{\ell=0}^{L-1}
\frac{s_\ell}{t_\ell^2}
=\Delta^2 A.
\end{aligned}
\end{equation}

\paragraph{Gaussian protection and bounded reconstruction noise.} For a random variable $X$, let $\mathcal{L}(X)$ denote the distribution of $X$.
 Let $G\sim\mathcal N(0,r^2I_{S_\star})$ and $P_D^j:=\mathcal L(F_j(D)+G)$.

The preceding sensitivity bound and the prescribed choice of $r$
ensure that the ordinary Gaussian mechanism
$P_D^j$ is
$(\varepsilon/4,\delta/16)$-differentially private for every fixed $j$ (see Lemma~\ref{lem:privacy:gaussian}).

Define  $
Z=\operatorname{Clip}_{1/4}(G)$ and $
Q_D^j:=\mathcal L(F_j(D)+Z)$. 
Under this coupling, the Gaussian and clipped-noise outputs agree
unless some coordinate of $G$ leaves $[-1/4,1/4]$. Hence
Lemma~\ref{lem:privacy:bounded} gives the uniform comparison
\begin{equation}
\label{eq:privacy:tv}
\operatorname{TV}(P_D^j,Q_D^j)
\le
\Pr\!\left(\max_{a\le S_\star}|G_a|>\frac14\right)
\le\eta.
\end{equation}
Replacing the two Gaussian laws in the privacy inequality by their
clipped-noise counterparts costs at most
$(1+e^{\varepsilon/4})\eta$ in the additive privacy parameter.
Consequently,
\begin{equation}
\label{eq:privacy:reconstruction}
D\longmapsto F_j(D)+Z
\quad\text{is}\quad
(\varepsilon/4,\delta/8)\text{-differentially private},
\end{equation}
uniformly over all fixed $j$ (see Lemma~\ref{lem:privacy:bounded}).

\paragraph{Composition and postprocessing.}
Selection of $J$ does not inspect the fresh reconstruction noises. Therefore,
after the first-stage transcript equals $j$, the second stage executes
exactly the fixed-label mechanism just analyzed. Adaptive composition
gives privacy for the auxiliary transcript
\begin{equation}
\label{eq:privacy:transcript}
\mathcal T(D):=(J,F_J(D)+Z),
\qquad
\mathcal T
\text{ is }
(3\varepsilon/8,5\delta/32)\text{-differentially private}.
\end{equation}
Writing $V=F_J(D)+Z$, the final matrix is recovered as
\begin{equation}
\label{eq:privacy:postprocessing}
\widehat\Sigma
=
\sum_{\ell=0}^{L-1}
\sum_{c=(e,\ell,h)\in J_\ell}
\frac{U_c}{\kappa\rho t_\ell}\psi(V_c)B_e.
\end{equation}
This map uses only the protected transcript $(J,V)$ and public quantities. Thus the matrix
inherits the privacy guarantee from $(J,V)$. Finally, the bound is
uniform over the public activations and 
 thresholds $\{U_c, T_c\}_{c\in \mathcal{C}}$, so averaging over
their data-independent law preserves it, even if those public variables
are revealed. Lemma~\ref{lem:privacy:composition} formalizes these
steps. Since $3\varepsilon/8\le\varepsilon$ and
$5\delta/32\le\delta$, the mechanism in Algorithm~\ref{alg:main} is
$(\varepsilon,\delta)$-differentially private on every input dataset.

%% file: dis.tex
\section{Discussion}\label{sec:dis}

We developed a multiscale randomized-threshold mechanism for private
sparse covariance estimation. The proposed mechanism achieves operator-norm error
$\alpha\sigma^2$ with probability at least $1-\beta$ using
$\widetilde O(k^2/\alpha^2+k\sqrt d/(\alpha\varepsilon))$ samples
under the sub-Gaussian assumption, while satisfying $(\varepsilon,\delta)$-differential privacy
under replace-one adjacency for arbitrary input datasets.
Our result removes the additional $\sqrt{k}$ factor from the
previous privacy term, matching the lower bound of
\citet{kumar2026curse} up to logarithmic factors in its applicable parameter regime. 
The key insight is to represent covariance entries through
noise-corrected, randomly weighted threshold tests, enabling
direct control of reconstruction fluctuations at the matrix
level without requiring exact support recovery.

Three directions merit further investigation. First, there is still a gap between the upper and lower bounds when $\delta >\varepsilon^2/d^2$. It remains open whether the lower bound of $k^2/\alpha^2 + k\sqrt{d}/(\alpha\varepsilon)$ still applies in this regime. 
Second, extending the mechanism to approximately sparse covariance
matrices would broaden its applicability. This extension would
require replacing the exact-support argument with quantitative
bounds on the bias and variance contributed by small but nonzero
entries.
Thirdly, a direct privacy analysis of the final aggregated matrix
may yield tighter guarantees than the current analysis, which
establishes privacy for the more informative auxiliary transcript
containing the selected labels and perturbed queries.
Understanding how aggregation conceals intermediate information
could enable sharper noise calibration, potentially improving
the logarithmic factors or practical accuracy of the mechanism.

%% file: appendix.tex
% Detailed proofs. Keep this order: the sample-size verification closes the
% assumptions used by privacy, coupling, and ideal matrix concentration.

Throughout the appendix, for any candidate test $c=(e,\ell,h)$ and
any candidate-indexed quantity $X$, we use the notations $X_c$ and
$X_{e,\ell,h}$ interchangeably.

\section{Additional Related Works}\label{app:rel}

\paragraph{Structured covariance estimation.}
Non-private research studies how covariance structure reduces statistical
complexity and which rates are minimax optimal.
For sparse covariance matrices, hard, generalized, and adaptive
thresholding provide operator-norm guarantees
\citep{bickel2008thresholding,rothman2009generalized,cai2011adaptive}.
\citet{cai2012optimal} establish minimax rates under specified sparsity
and high-dimensional conditions, while \citet{cai2016estimating} survey
structured covariance and precision-matrix estimation.
Operator-norm guarantees control the aggregate effect of estimation
errors across entries, rather than treating entries as unrelated scalar
estimation problems.
Here, sparsity concerns the population covariance, not individual
observations; sparse covariance therefore need not imply sparse empirical
outer products or low sensitivity.

\paragraph{Private matrix release.}
This direction studies empirical covariance release under bounded-input
assumptions, with guarantees depending on the loss and privacy definition.
\citet{amin2019differentially} study pure-private covariance release with
Frobenius-error guarantees.
\citet{dong2022differentially} obtain trace-sensitive and tail-sensitive
Frobenius-error bounds under concentrated differential privacy.
\citet{d2025purely} study pure-private empirical covariance estimation,
including spectral-norm guarantees.
The distinction between empirical matrix release and population
estimation also matters, since the latter includes sampling error.
Comparisons must preserve input normalization: bounds for unit-norm
records do not directly translate into dimension-independent sample
complexity for population covariance estimation to operator-norm error
$\alpha\sigma^2$ under a sub-Gaussian model.

\paragraph{Private distribution learning and stable estimation.}
A complementary direction estimates distributional parameters without
strong prior bounds on Gaussian parameters.
\citet{kamath2019privately} use recursive private preconditioning for
high-dimensional distribution learning, including Gaussian learning in
total variation distance.
\citet{brown2023fast} develop affine-invariant private mean and covariance
estimators for sub-Gaussian distributions using stable estimation.
These approaches address general distributional estimation rather than
specifically exploiting unknown row-sparse covariance support.
Here, the sub-Gaussian scale is known, the covariance support is unknown,
and the loss is absolute operator norm.

\paragraph{Private sparse estimation and lower bounds.}
\citet{wang2021differentially} develop private thresholding for high-dimensional
sparse covariance estimation.
Under the same sub-Gaussian normalization as this paper,
\citet{kumar2026curse} combine private row-wise selection with noisy entry
reconstruction, obtaining
$\widetilde{O}(k^2/\alpha^2+k^{3/2}\sqrt{d}/(\alpha\varepsilon))$
sample complexity. On the other side,
their privacy lower bound of $\tilde{\Omega}( k\sqrt{d}/(\alpha \varepsilon))$
draws on techniques related to 
\citet{narayanan2024better}, while \citet{cai2012optimal} provide the
non-private benchmark.
Our rate matches this lower bound up to logarithmic factors when 
$\log d\lesssim  k \lesssim d^{1/2-\lambda}$ for some constant $\lambda >0$ and $\delta\leq \epsilon^2/d^2$.
These additional restrictions are not required for the upper-bound
guarantee.

\paragraph{Private selection.}
Private top-$k$ selection is studied by \citet{durfee2019practical} and
\citet{qiao2021oneshot}; the latter provide the oneshot Laplace guarantee
used here.
Our mechanism selects threshold-test labels under public, scale-dependent
quotas rather than recovering the population support.
These quotas bound the sensitivity of the reconstruction queries for
every fixed selected label set on arbitrary datasets.

\section{Supporting lemmas for the approximation error}
\label{sec:accuracy:lemmas}

All notation is inherited from the main argument. For $c = (e,\ell,h)$, define
\[
W_c
:=
\frac{U_c}{\kappa\rho t_\ell}
\psi(S_e/T_c+Z_c)B_e,
\qquad
\overline{W}_c:=W_c-\mathbb E[W_c\mid S],
\qquad c=(e,\ell,h).
\]
In particular,
$\Sigma^{\mathrm{id}}=\sum_{c\in\mathcal C}W_c$ and
$\Sigma^{\mathrm{id}}-\overline\Sigma(S)=\sum_{c\in\mathcal C}\overline{W}_c$.

\begin{lemma}[Empirical covariance concentration]
\label{lem:accuracy:empirical}
There is a universal constant $C$ such that
\[
n\ge C\frac{k^2}{\alpha^2}\log\frac{80d}{\beta}
\quad\Longrightarrow\quad
\Pr(\mathcal E_{\mathrm{cov}}^c)\le\frac{\beta}{10}.
\]
\end{lemma}

\begin{proof}
Recall $R=\sigma\sqrt{2\log(40nd/\beta)}$. The coordinatewise
sub-Gaussian tail bound and a union bound give
\[
\Pr\!\left(\max_{a,j}|X_{aj}|>R\right)
\le
2nd\exp\!\left(-\frac{R^2}{2\sigma^2}\right)
=\frac{\beta}{20}.
\]
On the complementary event $\mathcal E_{\mathrm{clip}}$, clipping
does not alter any record, and hence
\[
S=S^{\mathrm{raw}}
:=\frac{1}{n}\sum_{a=1}^n X_aX_a^\top.
\]

For fixed $i,j$, the centered products
$X_{ai}X_{aj}-\Sigma_{ij}$ are independent across records and have
sub-exponential scale bounded by $C\sigma^2$. Indeed, sub-Gaussian
moment bounds and Cauchy--Schwarz give, for every integer $p\ge1$,
\[
\mathbb E|X_{ai}X_{aj}|^p
\le
\left(
\mathbb E|X_{ai}|^{2p}\,
\mathbb E|X_{aj}|^{2p}
\right)^{1/2}
\le(C\sigma^2p)^p,
\]
and centering changes this bound by at most a universal constant.
Then, by Bernstein's inequality (Lemma~\ref{lem:scalar-bernstein}),
\[
\Pr\!\left(
|S^{\mathrm{raw}}_{ij}-\Sigma_{ij}|>\tau
\right)
\le
2\exp\!\left[
-cn\min\left\{
\frac{\tau^2}{\sigma^4},\frac{\tau}{\sigma^2}
\right\}
\right].
\]
Taking $\tau=t_0/16=\alpha\sigma^2/(4096k)$ and applying a union
bound over at most $d^2$ entries gives
\[
\Pr\!\left(
\|S^{\mathrm{raw}}-\Sigma\|_{\max}>\frac{t_0}{16}
\right)
\le
2d^2\exp\!\left(-c'n\frac{\alpha^2}{k^2}\right)
\le\frac{\beta}{20}
\]
under the stated sample-size condition.

Finally,
\[
\mathcal E_{\mathrm{cov}}^c
\subseteq
\mathcal E_{\mathrm{clip}}^c
\cup
\left\{
\|S^{\mathrm{raw}}-\Sigma\|_{\max}>\frac{t_0}{16}
\right\}.
\]
Adding the two failure probabilities proves the claim.
\end{proof}

\begin{lemma}
\label{lem:accuracy:mean} Recall that $\mathcal S=\{e\in \mathcal{E}_d:\Sigma_e\ne0\}$. 
For every fixed $S$ satisfying $\mathcal E_{\mathrm{cov}}$,
\[
\overline\Sigma_e(S)
=
\frac{1}{\kappa}
\int_{t_0}^{t_L}
\mathbb E_Z\psi(S_e/t+Z)\,dt,
\qquad
\|\overline\Sigma(S)-S\|_{\max}\le t_0.
\]
Moreover, if $e\notin\mathcal S$, every ideal contribution at position
$e$ vanishes, so $
\Sigma^{\mathrm{id}}_e=\overline\Sigma_e(S)=0$.
In particular,
\[
\|\overline\Sigma(S)-\Sigma\|_{\mathrm{op}}
\le\frac{17}{16}kt_0.
\]
\end{lemma}

\begin{proof}
The symmetry and bounded support of $\nu_r$ imply
\[
1\le\kappa=\mathbb E(1-Z)^{-1}\le\frac{4}{3}.
\]
The lower bound follows from Jensen's inequality and
$\mathbb E Z=0$; the upper bound follows from $Z\le1/4$.

For $a>0$, the negative gate is impossible, while
\[
\psi(a/t+Z)
=
\mathbf 1\!\left\{t<\frac{a}{1-Z}\right\}.
\]
Integrating first over $t$ gives
\[
\int_0^\infty\mathbb E_Z\psi(a/t+Z)\,dt
=
\mathbb E_Z\frac{a}{1-Z}
=a\kappa.
\]
For $a=-b<0$, the positive gate is impossible, and the integral equals
$-b\,\mathbb E(1+Z)^{-1}=a\kappa$ by symmetry. The case $a=0$
is immediate. Thus
\begin{equation}
\label{eq:accuracy:scalar-identity}
\frac{1}{\kappa}
\int_0^\infty\mathbb E_Z\psi(a/t+Z)\,dt=a,
\qquad a\in\mathbb R.
\end{equation}

For the uniform-threshold construction, independence within a
candidate and $m_\ell p_\ell=\rho t_\ell^2$ give
\begin{align*}
&\mathbb E\!\left[
\left.
\sum_{h=1}^{m_\ell}
\frac{U_{e,\ell,h}}{\kappa\rho t_\ell}
\psi(a/T_{e,\ell,h}+Z_{e,\ell,h})
\,\right|\,S
\right]
=
\frac{m_\ell p_\ell}{\kappa\rho t_\ell^2}
\int_{t_\ell}^{2t_\ell}
\mathbb E_Z\psi(a/t+Z)\,dt
\\
&\qquad=
\frac{1}{\kappa}
\int_{t_\ell}^{2t_\ell}
\mathbb E_Z\psi(a/t+Z)\,dt.
\end{align*}
Taking $a=S_e$ and summing over the dyadic scales proves the
conditional-mean formula.

On $\mathcal E_{\mathrm{cov}}$,
\[
|S_e|\le\sigma^2+\frac{t_0}{16},
\qquad
\frac{4}{3}|S_e|<2\sigma^2\le t_L.
\]
A nonzero gate requires $t<4|S_e|/3$, so the integrand vanishes for
$t\ge t_L$. Subtracting~\eqref{eq:accuracy:scalar-identity} from
the conditional-mean formula yields
\[
\overline\Sigma_e(S)-S_e
=
-\frac{1}{\kappa}
\int_0^{t_0}\mathbb E_Z\psi(S_e/t+Z)\,dt.
\]
Since $|\psi|\le1$ and $\kappa\ge1$,
\[
|\overline\Sigma_e(S)-S_e|
\le\frac{t_0}{\kappa}\le t_0.
\]

If $\Sigma_e=0$, then $|S_e|\le t_0/16$, and every candidate at
that position satisfies
\[
|S_e/T_c+Z_c|
\le\frac{1}{16}+\frac{1}{4}<1.
\]
Every such contribution is zero. Hence
$\overline\Sigma(S)-\Sigma$ is supported on $\mathcal S$ and has
at most $k$ nonzero entries per row. Its entries satisfy
\[
|\overline\Sigma_e(S)-\Sigma_e|
\le
|\overline\Sigma_e(S)-S_e|+|S_e-\Sigma_e|
\le t_0+\frac{t_0}{16}.
\]
The row-sum bound for a symmetric matrix now gives
\[
\|\overline\Sigma(S)-\Sigma\|_{\mathrm{op}}
\le k\left(t_0+\frac{t_0}{16}\right)
=\frac{17}{16}kt_0.
\]
\end{proof}

\begin{lemma}[Conditional matrix concentration]
\label{lem:accuracy:matrix}
For every fixed $S$ satisfying $\mathcal E_{\mathrm{cov}}$ and
every $u>0$,
\[
\Pr\!\left(
\left.
\|\Sigma^{\mathrm{id}}-\overline\Sigma(S)\|_{\mathrm{op}}
>
\sqrt{\frac{2kLu}{\rho}}+\frac{4u}{3\rho t_0}
\,\right|\,S
\right)
\le2de^{-u}.
\]
In particular, for $u=\log(40d/\beta)\le H$ and the specified
public parameters,
\[
\Pr\!\left(
\left.
\|\Sigma^{\mathrm{id}}-\overline\Sigma(S)\|_{\mathrm{op}}
>\frac{\alpha\sigma^2}{2}
\,\right|\,S
\right)
\le\frac{\beta}{20}.
\]
\end{lemma}

\begin{proof}
Conditional on $S$, each $W_c$ depends only on its own independent
triple $(U_c,T_c,Z_c)$. Thus the centered matrices $\overline{W}_c$ are
independent and self-adjoint, with $\mathbb E[\overline{W}_c\mid S]=0$.

Fix $e$, $\ell$  and consider $c = (e,\ell,h)$ over different $h$. Note that $U_c^2=U_c$ and $\psi^2\le1$
imply
\begin{equation}
\label{eq:accuracy:one-scale-moment}
\begin{aligned}
\sum_{h=1}^{m_\ell}
\mathbb E[W_{e,\ell,h}^{\,2}\mid S]
&\preceq
\frac{m_\ell p_\ell}{\kappa^2\rho^2t_\ell^2}B_e^2
=
\frac{1}{\kappa^2\rho}B_e^2
\preceq\frac{1}{\rho}B_e^2.
\end{aligned}
\end{equation}
By Lemma~\ref{lem:accuracy:mean}, positions outside $\mathcal S$
contribute nothing. Direct multiplication gives
\[
B_{ii}^2=e_ie_i^\top,
\qquad
B_{ij}^2=e_ie_i^\top+e_je_j^\top
\quad(i<j).
\]
Consequently, $
\sum_{e\in\mathcal S}B_e^2\preceq kI_d$, 
because each diagonal entry of this sum counts the nonzeros in the
corresponding row of $\Sigma$. Summing
\eqref{eq:accuracy:one-scale-moment} over positions and scales yields
\[
\sum_{c\in\mathcal C}\mathbb E[W_c^2\mid S]
\preceq\frac{kL}{\rho}I_d.
\]

Centering cannot increase the second-moment matrix:
\[
\mathbb E[\overline{W}_c^2\mid S]
=
\mathbb E[W_c^2\mid S]
-\bigl(\mathbb E[W_c\mid S]\bigr)^2
\preceq\mathbb E[W_c^2\mid S].
\]
Thus the conditional matrix variance satisfies
\[
v:=
\left\|
\sum_{c\in\mathcal C}\mathbb E[\overline{W}_c^2\mid S]
\right\|_{\mathrm{op}}
\le\frac{kL}{\rho}.
\]

Since $\|B_e\|_{\mathrm{op}}=1$, $t_\ell\ge t_0$, and
$\kappa\ge1$,
\[
\|W_c\|_{\mathrm{op}}\le\frac{1}{\rho t_0},
\qquad
\|\overline{W}_c\|_{\mathrm{op}}\le\frac{2}{\rho t_0}=:B.
\]
The self-adjoint matrix Bernstein inequality
\cite[Theorem~6.1]{tropp2012user}, applied to the sum and its negative,
gives
\[
\Pr\!\left(
\left.
\left\|\sum_{c\in\mathcal C}\overline{W}_c\right\|_{\mathrm{op}}
>x
\,\right|\,S
\right)
\le
2d\exp\!\left(-\frac{x^2}{2(v+Bx/3)}\right).
\]
For $x=\sqrt{2vu}+2Bu/3$, one has
$x^2\ge2u(v+Bx/3)$, so
\[
\Pr\!\left(
\left.
\left\|\sum_{c\in\mathcal C}\overline{W}_c\right\|_{\mathrm{op}}
>\sqrt{2vu}+\frac{2Bu}{3}
\,\right|\,S
\right)
\le2de^{-u}.
\]
Substituting the bounds on $v$ and $B$ proves the first claim.
If $v=0$, all centered summands vanish almost surely and the
conclusion is immediate.

For $u=\log(40d/\beta)\le H$, recall
$t_0=\alpha\sigma^2/(256k)$ and
$\rho=C_\rho kLH/(\alpha^2\sigma^4)$. Then
\begin{align*}
\sqrt{\frac{2kLu}{\rho}}+\frac{4u}{3\rho t_0}
&=
\alpha\sigma^2
\left(
\sqrt{\frac{2u}{C_\rho H}}
+\frac{1024u}{3C_\rho LH}
\right)
\le
\alpha\sigma^2
\left(
\sqrt{\frac{2}{C_\rho}}
+\frac{1024}{3C_\rho L}
\right)
\le\frac{1}{2}\alpha\sigma^2,
\end{align*}
where the last inequality uses $C_\rho=4096$ and $L\ge11$.
Since $2de^{-u}=\beta/20$, the second claim follows.
\end{proof}

\begin{lemma}[Coupling lemma]
\label{lem:accuracy:coupling}
Under the sample-size condition of the main theorem, the actual and
ideal estimators admit a coupling such that
\[
\Pr\!\left(
\mathcal E_{\mathrm{cov}}
\cap\{\widehat\Sigma\ne\Sigma^{\mathrm{id}}\}
\right)
\le\frac{\beta}{20}.
\]
\end{lemma}

\begin{proof}
Generate the reconstruction noises $Z_c$ for all candidates  $c\in \mathcal{C}$ before
selection, and use the same noise on each selected label in both
estimators. Selection does not inspect these noises, so this
construction preserves the distribution of the actual mechanism.

Define
\[
\mathcal R_\ell=\{c\in\mathcal C_\ell:q_c\ge-3/8\},
\qquad
K_\ell=\{e\in \mathcal{E}_d:|\Sigma_e|\ge t_\ell/2\},
\]
\[
B_\ell=\sum_{e\in K_\ell}\sum_{h=1}^{m_\ell}U_{e,\ell ,h},
\qquad
D_\ell=d\rho\min\{kt_\ell^2,\sigma^4\}.
\]
If $c=(e,\ell,h)\in\mathcal R_\ell$, then $U_c=1$ and
$|S_e|\ge5T_c/8$. On $\mathcal E_{\mathrm{cov}}$,
\[
|\Sigma_e|
\ge\frac{5}{8}t_\ell-\frac{t_0}{16}
\ge\frac{9}{16}t_\ell>\frac{t_\ell}{2}.
\]
Hence $e\in K_\ell$ and, by summing the corresponding indicator
inequality, $|\mathcal R_\ell|\le B_\ell$.

The $k$-sparsity and
$\sum_j\Sigma_{ij}^2\le\sigma^4$ give $
|K_\ell|
\le d\min\left\{k,\frac{4\sigma^4}{t_\ell^2}\right\}$. 
For the uniform-threshold construction,
$m_\ell p_\ell=\rho t_\ell^2$, so
\[
\mathbb E B_\ell
=|K_\ell|m_\ell p_\ell
\le d\rho\min\{kt_\ell^2,4\sigma^4\}
\le4D_\ell.
\]
Since $K_\ell$ is fixed once $\Sigma$ is fixed, $B_\ell$ is a
sum of independent Bernoulli variables. Therefore,
\[
\mathbb E e^{B_\ell}
\le\exp\!\bigl((\mathrm e-1)\mathbb E B_\ell\bigr).
\]
For a layer with $s_\ell<M_\ell$, the quota satisfies
$s_\ell\ge64(D_\ell+H)$, and exponential Markov gives
\begin{align*}
\Pr(B_\ell>s_\ell)
&\le
\exp\!\bigl(-s_\ell+(\mathrm e-1)\mathbb E B_\ell\bigr)
\\
&\le
\exp\!\bigl(-64(D_\ell+H)+4(\mathrm e-1)D_\ell\bigr)
\le e^{-H}.
\end{align*}
If $s_\ell=M_\ell$, then $B_\ell\le s_\ell$ deterministically.
Thus, defining $
\mathcal E_{\mathrm{quota}}
=\bigcap_{\ell=0}^{L-1}\{B_\ell\le s_\ell\}$, 
a union bound gives
\[
\Pr(\mathcal E_{\mathrm{quota}}^c)
\le Le^{-H}=\frac{\beta}{80d}\le\frac{\beta}{40}.
\]
On $\mathcal E_{\mathrm{cov}}\cap\mathcal E_{\mathrm{quota}}$,
we consequently have $|\mathcal R_\ell|\le s_\ell$ at every scale.

Recall that $M = \sum_{\ell=0}^{L-1} M_{\ell}$ and $\Gamma = \log(80 LM/(\beta\delta))$. 
Then $b_\ell\le1/(64\Gamma)$ for any $0\leq \ell \leq L-1$. The Laplace tail bound therefore implies
\[
\Pr(|\zeta_c|>1/32)
=\exp\!\left(-\frac{1}{32b_\ell}\right)
\le e^{-2\Gamma}.
\]
There are at most $M$ generated selection noises. For the event
\[
\mathcal E_{\mathrm{noise}}
=
\{\,|\zeta_c|\le1/32
  \text{ for every } c\in C\,\},
\]
by Lemma~\ref{lem:selection-noise-bound}, we obtain $b_{\ell} \leq \frac{1}{64\Gamma}$ and
\[
\Pr(\mathcal E_{\mathrm{noise}}^c)
\le Me^{-2\Gamma}\le\frac{\beta}{40}.
\]

Now take an active candidate with $|S_e|/T_c>3/4$. It is relevant,
and on $\mathcal E_{\mathrm{noise}}$ its noisy score satisfies
\[
q_c+\zeta_c>-\frac{1}{4}-\frac{1}{32}=-\frac{9}{32}.
\]
By contrast, every nonrelevant candidate satisfies
\[
q_{c'}+\zeta_{c'}<-\frac{3}{8}+\frac{1}{32}=-\frac{11}{32}.
\]
Only other relevant candidates can precede $c$ in the noisy ranking.
Since there are at most $s_\ell$ relevant candidates in total,
including $c$, its rank is at most $s_\ell$, and it must be selected.
If the layer retains all candidates, this conclusion is immediate.

Every unselected active candidate therefore satisfies
$|S_e|/T_c\le3/4$, and hence
\[
|S_e/T_c+Z_c|
\le\frac{3}{4}+\frac{1}{4}=1.
\]
Its gate is zero, and inactive candidates contribute zero through
$U_c=0$. Under the shared-noise coupling,
\[
\Sigma^{\mathrm{id}}-\widehat\Sigma
=
\sum_{\ell=0}^{L-1}
\sum_{c\in\mathcal C_\ell\setminus J_\ell}W_c
=0
\]
on
$\mathcal E_{\mathrm{cov}}\cap\mathcal E_{\mathrm{quota}}
\cap\mathcal E_{\mathrm{noise}}$.
Consequently,
\[
\mathcal E_{\mathrm{cov}}
\cap\{\widehat\Sigma\ne\Sigma^{\mathrm{id}}\}
\subseteq
\mathcal E_{\mathrm{quota}}^c
\cup\mathcal E_{\mathrm{noise}}^c,
\]
and the desired probability bound follows by a union bound.
\end{proof}

\begin{lemma}
\label{lem:selection-noise-bound} For every $0 \le \ell < L$ with $s_\ell < M_\ell$, it holds that $
    b_\ell \le \frac{1}{64\Gamma}$.
\end{lemma}

\begin{proof}
Since $H \ge 1$, the  definition of $s_{\ell}$ gives $
    s_\ell
    \le \left\lceil 64(d\rho k t_\ell^2 + H) \right\rceil
    \le 65(d\rho k t_\ell^2 + H)$.
Using $t_\ell \ge t_0$, the definitions of $t_0$ and $\rho$,
and $dL \ge 22$, we obtain
\[
    \frac{s_\ell}{t_\ell^2}
    \le 65\left(d\rho k + \frac{H}{t_0^2}\right)
    = 65d\rho k\left(1+\frac{16}{dL}\right)
    \le 130d\rho k.
\]
Moreover, $
    \log\frac{M_\ell}{\delta'}
    = \log\frac{32LM_\ell}{\delta}
    \le \Gamma.$
Substituting $
    \varepsilon' = \frac{\varepsilon}{8L}$,$
    \Delta = \frac{4\sigma^2}{n}\log\frac{40nd}{\beta}$ and $
    \rho = \frac{4096LkH}{\alpha^2\sigma^4} $
into the definition of $b_\ell$ yields
\begin{align*}
    b_\ell
    &= \frac{64L\Delta}{\varepsilon}
       \sqrt{\frac{s_\ell}{t_\ell^2}
             \log\frac{M_\ell}{\delta'}} \le \frac{64L\Delta}{\varepsilon}
          \sqrt{130d\rho k\Gamma} \le C_0\,
          \frac{k\sqrt{d}}{n\alpha\varepsilon}
          \log\frac{80nd}{\beta}\,
          L^{3/2}\sqrt{H\Gamma},
\end{align*}
where $C_0>0$ is universal. For a sufficiently large universal constant \(C\) in \eqref{eq:sample}, we have $
    b_\ell
    \le \frac{1}{64\Gamma}$.
\end{proof}

\begin{lemma}[Bernstein's inequality for sub-exponential variables]
\label{lem:scalar-bernstein}
Let $X_1,\ldots,X_n$ be independent, mean-zero, real-valued random
variables satisfying $
\|X_i\|_{\psi_1}\le K$
for all $i$,
where $
\|X\|_{\psi_1}
:=
\inf\left\{
s>0:\mathbb E\exp(|X|/s)\le2
\right\}$. 
There exists a universal constant $c>0$ such that, for every $t>0$,
\[
\Pr\!\left(
\left|\sum_{i=1}^n X_i\right|\ge t
\right)
\le
2\exp\!\left[
-c\min\left\{
\frac{t^2}{nK^2},\frac{t}{K}
\right\}
\right].
\]
Equivalently, for every $\tau>0$,
\[
\Pr\!\left(
\left|\frac1n\sum_{i=1}^n X_i\right|\ge\tau
\right)
\le
2\exp\!\left[
-cn\min\left\{
\frac{\tau^2}{K^2},\frac{\tau}{K}
\right\}
\right].
\]
Consequently, there is a universal constant $C>0$ such that, for
every $\eta\in(0,1)$, with probability at least $1-\eta$,
\[
\left|\frac1n\sum_{i=1}^n X_i\right|
\le
CK\left(
\sqrt{\frac{\log(2/\eta)}{n}}
+\frac{\log(2/\eta)}{n}
\right).
\]
\end{lemma}

\section{Supporting lemmas for the privacy guarantee}
\label{app:privacy}

The arguments below fix the public activations and thresholds until
the final averaging step. The notation and public parameters are the
same as in Section~\ref{sec:privacy}.

\begin{lemma}[Clipped-entry sensitivity and private selection]
\label{lem:privacy:selection}
For every neighboring pair $D\sim D'$ and every $e\in \mathcal{E}_d$,
\[
|S_e(D)-S_e(D')|\le\Delta.
\]
For every fixed realization of the public activations and thresholds,
the selected label collection $J$ is
$(\varepsilon/8,\delta/32)$-differentially private.
\end{lemma}

\begin{proof}
Suppose record $a$ is replaced. For $e=(i,j)$, all other terms in the
empirical average cancel, and hence
\[
S_e(D)-S_e(D')
=
\frac{Y_{ai}Y_{aj}-Y'_{ai}Y'_{aj}}{n}.
\]
Each clipped coordinate lies in $[-R,R]$, so each product lies in
$[-R^2,R^2]$. It follows that
\[
|S_e(D)-S_e(D')|
\le\frac{2R^2}{n}=\Delta.
\]

For $c = (e,\ell,h)$,
\[
q_c(D)=
\left(\frac{|S_{e}(D)|}{T_c}-1\right)U_c
-2(1-U_c).
\]
For an active candidate, the triangle inequality and
$T_c\ge t_\ell$ give
\[
|q_c(D)-q_c(D')|
=
\frac{\bigl||S_{e}(D)|-|S_{e}(D')|\bigr|}{T_c}
\le
\frac{|S_{e}(D)-S_{e}(D')|}{T_c}
\le\frac{\Delta}{t_\ell}.
\]
For an inactive candidate, both scores equal $-2$. Therefore the
whole score vector at scale $\ell$ has $\ell_\infty$ sensitivity at
most $\Delta/t_\ell$, even though many of its coordinates may change.

We use the one-shot Laplace calibration invoked in \citep[Theorem~2.2]{qiao2021oneshot}:
for a score vector of $\ell_\infty$ sensitivity $g$, independent
Laplace perturbations of scale
\[
\frac{8g\sqrt{s\log(M/\delta_0)}}{\varepsilon_0}
\]
suffice to select an unordered top-$s$ label set with
$(\varepsilon_0,\delta_0)$-differential privacy, in the stated
parameter range $\varepsilon_0\le0.2$, $\delta_0\le0.05$, and
$M\ge2$. 
Selecting the largest scores is equivalent to selecting the smallest negated scores; only the unordered set of selected labels is released.

The manuscript's parameter range gives $L\ge11$, so
\[
\varepsilon'=\frac{\varepsilon}{8L}
\le\frac1{88}<0.2,
\qquad
\delta'=\frac{\delta}{32L}
\le\frac1{3520}<0.05,
\qquad
M_\ell\ge\frac{d(d+1)}2\ge3.
\]
For a scale with $s_\ell<M_\ell$, substituting
$g=\Delta/t_\ell$, $s=s_\ell$, and $M=M_\ell$ yields the prescribed
noise scale
\[
b_\ell
=
\frac{8(\Delta/t_\ell)
\sqrt{s_\ell\log(M_\ell/\delta')}}{\varepsilon'}.
\]
Thus this scale is $(\varepsilon',\delta')$-differentially private.
A scale with $s_\ell=M_\ell$ returns all labels independently of the
data and has zero privacy cost for the selection. Basic composition over at most $L$
scales gives
\[
(L\varepsilon',L\delta')
=(\varepsilon/8,\delta/32).
\]
The selection noises retain their full Laplace law throughout; the
small-noise event used for accuracy is not imposed here.
\end{proof}

\begin{lemma}
\label{lem:privacy:query}
For every fixed label collection $j = \{ j_{\ell}\}_{\ell=0}^{L-1}$ with $|j_\ell|=s_\ell$, every
fixed realization of the public activations and thresholds, and every
neighboring pair $D\sim D'$, the query $F_j$ satisfies
\[
\|F_j(D)-F_j(D')\|_2\le\Delta\sqrt A.
\]
\end{lemma}

\begin{proof}
Using Lemma~\ref{lem:privacy:selection} and the fact that $U_c\in\{0,1\}$ and
$T_c\ge t_\ell$ for any $c$, we obtain
\[
\begin{aligned}
\|F_j(D)-F_j(D')\|_2^2
&=
\sum_{\ell=0}^{L-1}\sum_{c =(e,\ell,h)\in j_\ell}
\frac{
U_c^2\bigl(S_{e}(D)-S_{e}(D')\bigr)^2
}{T_{c}^2}
\le
\sum_{\ell=0}^{L-1}\sum_{c\in j_\ell}
\frac{\Delta^2}{t_\ell^2}
\\
&=
\Delta^2\sum_{\ell=0}^{L-1}\frac{s_\ell}{t_\ell^2}
=\Delta^2 A.
\end{aligned}
\]
The proof is finished.
%Repeated tickets at the same matrix position are distinct coordinates of $F_j$ and are already counted in this sum. No independence of thequery values is assumed. The comparison keeps $j$ fixed on bothsides; it makes no assertion that separately selected query vectors have bounded sensitivity.
\end{proof}

\begin{lemma}
\label{lem:privacy:bounded}
For every fixed $j$ and every fixed realization of the public
activations and thresholds, let $P_D^j$ and $Q_D^j$ be the Gaussian
and clipped-noise output laws in Section~\ref{sec:privacy}. Then
\[
\operatorname{TV}(P_D^j,Q_D^j)\le\eta
\qquad\text{for every }D.
\]
Moreover, letting $Z_c\sim \nu_r$ for each $c \in \mathcal{C}$, the clipped-noise query mechanism $D\mapsto F_j(D)+Z$ is $
\left(
\varepsilon/4, \delta/8
\right)$\text{-differentially private}.

\end{lemma}

\begin{proof}
For a
query of $L_2$ sensitivity $s'$ and a privacy parameter
$\varepsilon_0\in(0,1]$, the condition $
r\ge
\frac{2s'}{\varepsilon_0}
\sqrt{\log\frac{2}{\delta_0}}$
is sufficient for $(\varepsilon_0,\delta_0)$-differential privacy.
This is a conservative form of the standard Gaussian-mechanism
calibration (see ~\citep{dwork2014algorithmic}).
Set $
s'=\Delta\sqrt A$, $
\varepsilon_0=\frac{\varepsilon}{4}$ and $\delta_0=\frac{\delta}{16}$. 
Then $r\geq \frac{2s'}{\varepsilon_0} \sqrt{\log\left(\frac{2}{\delta_0}\right)}$. Therefore, we learn that for every measurable
output event $F$ and every neighboring pair $D\sim D'$,
\begin{equation}
\label{eq:privacy:gaussian-base}
P_D^j(F)
\le e^{\varepsilon/4}P_{D'}^j(F)+\frac{\delta}{16}.
\end{equation}

Recall that $S_{\star} = \sum_{\ell=0}^{L-1} s_{\ell}$. 
 Generate one vector
$G\sim\mathcal N(0,r^2I_{S_\star})$ and set
$Z=\operatorname{Clip}_{1/4}(G)$ coordinatewise. Define
\[
A_G=\left\{\max_{a\le S_\star}|G_a|\le\frac14\right\}.
\]
On $A_G$, $Z=G$, and hence
$F_j(D)+Z=F_j(D)+G$ for every fixed $D$. For any measurable $F$,
\[
\begin{aligned}
|Q_D^j(F)-P_D^j(F)|
&\le
\mathbb E\left|
\mathbf1\{F_j(D)+Z\in F\}
-
\mathbf1\{F_j(D)+G\in F\}
\right|
\le\Pr(A_G^c).
\end{aligned}
\]
Taking the supremum over $F$, then applying a union bound and the
Gaussian tail bound, gives
\begin{equation}
\label{eq:privacy:clipping-coupling}
\begin{aligned}
\operatorname{TV}(P_D^j,Q_D^j)
&\le\Pr(A_G^c)
\le
\sum_{a=1}^{S_\star}
2\exp\!\left(-\frac{(1/4)^2}{2r^2}\right)
=
2S_\star\exp\!\left(-\frac1{32r^2}\right)
=\eta.
\end{aligned}
\end{equation}

Using~\eqref{eq:privacy:gaussian-base} and the total-variation bound
on both neighboring datasets, we obtain
\[
\begin{aligned}
Q_D^j(F)
&\le P_D^j(F)+\eta
\le
e^{\varepsilon/4}P_{D'}^j(F)
+\frac{\delta}{16}+\eta
\le
e^{\varepsilon/4}Q_{D'}^j(F)
+\frac{\delta}{16}
+(1+e^{\varepsilon/4})\eta.
\end{aligned}
\]
Noting that $(1+e^{\varepsilon/4}) \eta \leq \delta/16$, the last
two additive terms are at most $\delta/8$, which proves the asserted
$(\varepsilon/4,\delta/8)$ guarantee.
\end{proof}

\begin{lemma}[Privacy of the final output]
\label{lem:privacy:composition}
The auxiliary
transcript $(J,V)$ and the released matrix are
$(3\varepsilon/8,5\delta/32)$-differentially private for every fixed
public seed. Therefore, the mechanism in Algorithm~\ref{alg:main} is $(\varepsilon,\delta)$-differentially
private for every input dataset.
\end{lemma}

\begin{proof}
For every fixed public seed, Lemma~\ref{lem:privacy:selection}
protects the first-stage transcript $J$. For every possible value $j$
of that transcript, Lemma~\ref{lem:privacy:bounded} protects the
second-stage query with the same privacy parameters. Thus, after selection
returns $j$, the second-stage conditional mechanism is precisely
$F_j(D)+Z$, with independent coordinates of $Z$ having law $\nu_r$.
The uniform statement for all $j$, including values not realized on a
particular dataset, is the condition required by adaptive composition.

The adaptive basic composition theorem adds the two budgets
\cite{dwork2014algorithmic}, giving
\[
\left(
\frac{\varepsilon}{8}+\frac{\varepsilon}{4},
\frac{\delta}{32}+\frac{\delta}{8}
\right)
=
\left(\frac{3\varepsilon}{8},\frac{5\delta}{32}\right)
\]
for $(J,F_J(D)+Z)$. The auxiliary transcript is only a proof device
and need not actually be released.

For $V=F_J(D)+Z$, the reconstruction formula
\eqref{eq:privacy:postprocessing} uses only the transcript and public
quantities. For $c = (e,\ell,h)$, when $U_c=1$, the coordinate is
$V_c=S_{e}(D)/T_c+Z_c$, exactly as required by the algorithm.
When $U_c=0$, it is $V_c=Z_c\in[-1/4,1/4]$, so the gate is zero;
these dummy noise coordinates can be added for the proof and omitted
in the implementation.  Thus the output is
a data-independent postprocessing of the auxiliary transcript and
inherits its privacy guarantee.

To remove the conditioning on the public seed, let $F$ be any event
in the joint space of public seeds and released matrices, and let
$F_{U,T}$ denote its section at a fixed seed. The conditional privacy
bound is uniform, so
\[
\begin{aligned}
\Pr_D\bigl((U,T,\widehat\Sigma)\in F\bigr)
&=
\mathbb E_{U,T}\!\left[
\Pr_D\bigl(\widehat\Sigma\in F_{U,T}\mid U,T\bigr)
\right]
\\
&\le
\mathbb E_{U,T}\!\left[
e^{3\varepsilon/8}
\Pr_{D'}\bigl(\widehat\Sigma\in F_{U,T}\mid U,T\bigr)
+\frac{5\delta}{32}
\right]
\\
&=
e^{3\varepsilon/8}
\Pr_{D'}\bigl((U,T,\widehat\Sigma)\in F\bigr)
+\frac{5\delta}{32}.
\end{aligned}
\]
Therefore, the privacy guarantee holds even when all
public candidate randomness is revealed. The proof is finished since
$3\varepsilon/8\le\varepsilon$ and $5\delta/32\le\delta$, 
\end{proof}

\begin{lemma}
\label{lem:privacy:safety}
Recall $M=\sum_{\ell=0}^{L-1}M_\ell$ and
$\Gamma=\log(80LM/(\beta\delta))$. There is a universal constant $C$
such that
\begin{equation}
\label{eq:privacy:sample-condition}
n\ge
C\frac{k\sqrt d}{\alpha\varepsilon}
\log\frac{80nd}{\beta}\,
L^{3/2}\sqrt H\,\Gamma^{3/2}
\end{equation}
implies $r\le1/(8\sqrt\Gamma)$ and $(1+e^{\varepsilon/4})\eta\le\frac{\delta}{16}$. 
\end{lemma}

\begin{proof}
Recall $D_\ell=d\rho\min\{kt_\ell^2,\sigma^4\}$. Since $H\ge1$,
the ceiling and cap in the quota imply
\[
s_\ell
\le\lceil64(D_\ell+H)\rceil
\le65(D_\ell+H)
\le65(d\rho kt_\ell^2+H).
\]
Also, the geometric scale sequence gives
\[
\sum_{\ell=0}^{L-1}\frac1{t_\ell^2}
=
\frac1{t_0^2}\sum_{\ell=0}^{L-1}4^{-\ell}
\le\frac4{3t_0^2}.
\]
Consequently,
\begin{equation}
\label{eq:privacy:A-budget}
A
\le65d\rho kL+\frac{260H}{3t_0^2}
\le C\left(d\rho kL+\frac{H}{t_0^2}\right)
\le\frac{Cdk^2L^2H}{\alpha^2\sigma^4}.
\end{equation}
The last substitution uses
$\rho=C_\rho LkH/(\alpha^2\sigma^4)$ and
$t_0=\alpha\sigma^2/(256k)$. In particular, the additive quota
allowance $H$ contributes the term $H/t_0^2$ and is not discarded.

Substituting~\eqref{eq:privacy:A-budget} and $\Delta=2R^2/n$
into the prescribed Gaussian scale gives
\[
r
\le
C\frac{k\sqrt d}{n\alpha\varepsilon}
\frac{R^2}{\sigma^2}
L\sqrt H\sqrt{\log(32/\delta)}.
\]
Since
\[
\frac{R^2}{\sigma^2}
=2\log(40nd/\beta)
\le2\log(80nd/\beta),
\]
it suffices for $r\le1/(8\sqrt\Gamma)$ that
\[
n\ge
C\frac{k\sqrt d}{\alpha\varepsilon}
\log\frac{80nd}{\beta}\,
L\sqrt H\sqrt{\log(32/\delta)}\sqrt\Gamma.
\]
Now $\log(32/\delta)\le\Gamma$ and $L,\Gamma\ge1$, so $
\sqrt{\log(32/\delta)}\sqrt\Gamma
\le\Gamma
\le\sqrt L\,\Gamma^{3/2}$.

Finally, $r\le1/(8\sqrt\Gamma)$ and $S_\star\le M$ imply
\[
\eta
=2S_\star e^{-1/(32r^2)}
\le2Me^{-2\Gamma}
=\frac{\beta^2\delta^2}{3200L^2M}.
\]
Therefore,
\[
\frac{16(1+e^{\varepsilon/4})\eta}{\delta}
\le
\frac{\beta^2\delta(1+e^{\varepsilon/4})}{200L^2M}
<1
\]
under the stated parameter ranges.
\end{proof}

\begin{lemma}[Gaussian protection for fixed selected labels]
\label{lem:privacy:gaussian}
Fix the public activations and thresholds. For every fixed label
collection $j$, the mechanism with output law
\[
P_D^j
:=
\mathcal L\bigl(F_j(D)+G\bigr),
\qquad
G\sim\mathcal N(0,r^2I_{S_\star}),
\]
is $(\varepsilon/4,\delta/16)$-differentially private, where
\[
r=\frac{8\Delta\sqrt A}{\varepsilon}
\sqrt{\log\frac{32}{\delta}}.
\]
\end{lemma}

\begin{proof}
By Lemma~\ref{lem:privacy:query}, $F_j$ has $\ell_2$ sensitivity
at most $s_2=\Delta\sqrt A$. A sufficient Gaussian-mechanism
calibration for $(\varepsilon_0,\delta_0)$-differential privacy is
\[
r\ge
\frac{2s_2}{\varepsilon_0}
\sqrt{\log\frac{2}{\delta_0}},
\qquad 0<\varepsilon_0\le1;
\]
see \cite{dwork2014algorithmic}.
Taking $\varepsilon_0=\varepsilon/4$ and $\delta_0=\delta/16$,
the required noise scale is
\[
\frac{2\Delta\sqrt A}{\varepsilon/4}
\sqrt{\log\frac{2}{\delta/16}}
=
\frac{8\Delta\sqrt A}{\varepsilon}
\sqrt{\log\frac{32}{\delta}}
=r.
\]
The proof is finished.
\end{proof}